\documentclass{article}
\RequirePackage{relsize}
\AtBeginDocument{\larger[1]}

\usepackage[preprint]{neurips_2023}

\usepackage[utf8]{inputenc} 
\usepackage[T1]{fontenc}    
\usepackage{hyperref}       
\usepackage{url}            
\usepackage{booktabs}       
\usepackage{amsfonts}       
\usepackage{nicefrac}       
\usepackage{microtype}      
\usepackage{xcolor}         

\usepackage{amsmath, amsthm, amssymb}
\newtheorem{definition}{Definition}

\newtheorem{theorem}{Theorem}

\usepackage{caption}
\usepackage{graphicx}
\usepackage{subcaption}
\usepackage{multirow}

\usepackage{CJKutf8}
\usepackage{xcolor}
\usepackage{enumitem}

\usepackage[normalem]{ulem}

\title{Causal Inference as Distribution Adaptation: Optimizing ATE Risk under Propensity Uncertainty}

\author{
  Ashley Zhang\\
  \texttt{zhangyunzhe2023@berkeley.edu}\\
}

\begin{document}

\maketitle

\begin{abstract}
Standard approaches to causal inference, such as Outcome Regression and Inverse Probability Weighted Regression Adjustment (IPWRA), are typically derived through the lens of missing data imputation and identification theory. In this work, we unify these methods from a Machine Learning perspective, reframing ATE estimation as a \textit{domain adaptation problem under distribution shift}. We demonstrate that the canonical Hajek estimator is a special case of IPWRA restricted to a constant hypothesis class, and that IPWRA itself is fundamentally Importance-Weighted Empirical Risk Minimization designed to correct for the covariate shift between the treated sub-population and the target population.

Leveraging this unified framework, we critically examine the optimization objectives of Doubly Robust estimators. We argue that standard methods enforce \textit{sufficient but not necessary} conditions for consistency by requiring outcome models to be individually unbiased. We define the true "ATE Risk Function" and show that minimizing it requires only that the biases of the treated and control models structurally cancel out. Exploiting this insight, we propose the \textbf{Joint Robust Estimator (JRE)}. Instead of treating propensity estimation and outcome modeling as independent stages, JRE utilizes bootstrap-based uncertainty quantification of the propensity score to train outcome models jointly. By optimizing for the expected ATE risk over the distribution of propensity scores, JRE leverages model degrees of freedom to achieve robustness against propensity misspecification. Simulation studies demonstrate that JRE achieves up to a 15\% reduction in MSE compared to standard IPWRA in finite-sample regimes with misspecified outcome models.
\end{abstract}

\section{Introduction}

Estimating causal effects from observational data is a central challenge in econometrics, biostatistics, and increasingly, machine learning. Unlike Randomized Controlled Trials (RCTs), where treatment assignment is independent of background characteristics, observational studies are plagued by confounding: the treated and control populations often differ systematically in their covariate distributions. The fundamental goal of causal inference in this setting is to adjust for these differences to estimate the Average Treatment Effect (ATE), defined as \(\tau = \mathbb{E}[Y(1) - Y(0)]\), where \(Y(1)\) and \(Y(0)\) denote the potential outcomes under treatment and control, respectively.

Standard approaches to this problem generally fall into two categories: modeling the treatment assignment mechanism (e.g., Inverse Probability Weighting or IPW) or modeling the outcome surface (e.g., Outcome Regression or OR). While both strategies yield consistent estimates under correctly specified models and the assumption of strong ignorability, they suffer from complementary weaknesses. IPW estimators are notorious for high variance when propensity scores are close to zero or one, while Outcome Regression is sensitive to model misspecification. To mitigate these risks, Doubly Robust (DR) estimators, such as the Augmented IPW (AIPW) and Inverse Probability Weighted Regression Adjustment (IPWRA), have become the gold standard. These estimators require only one of the two nuisance models—the propensity score or the outcome model—to be correctly specified to achieve consistency.

However, from a Machine Learning (ML) perspective, the standard justification for these estimators often obscures the underlying mechanism of learning. In this paper, we reframe the problem of causal inference as a problem of \textit{distribution shift adaptation}. We argue that standard Outcome Regression fails not merely because of confounding, but because it performs Empirical Risk Minimization (ERM) on the conditional distribution of the treated units, \(P(X|Z=1)\), while the target inference requires minimizing risk over the full population distribution, \(P(X)\). Under this lens, IPWRA is not just a statistical adjustment but an importance-sampling technique that corrects this covariate shift, aligning the training distribution with the test distribution.

Furthermore, we revisit the optimization objective of Doubly Robust estimators. The standard theoretical derivation for IPWRA ensures that the outcome models, \(\mu_1(X)\) and \(\mu_0(X)\), are individually unbiased estimators of the potential outcomes \(Y(1)\) and \(Y(0)\). We demonstrate that while this condition is \textit{sufficient} to minimize the error of the ATE, it is not \textit{necessary}. The true objective of interest is minimizing the squared error of the difference, or the \(MSE_{ATE}\), defined as:
\begin{equation}
\label{eq:mse_ate}
\mathcal{R}_{ATE} = \mathbb{E}\left[ \left( (\mu_1(X) - Y(1)) - (\mu_0(X) - Y(0)) \right)^2 \right]
\end{equation}
Standard DR methods minimize Equation \eqref{eq:mse_ate} by forcing both error terms to zero independently. We argue that this approach leaves significant "degrees of freedom" on the table. In finite-sample regimes or under model misspecification, enforcing individual unbiasedness may consume model capacity that could otherwise be used to ensure that the biases of \(\mu_1\) and \(\mu_0\) are correlated and cancel each other out.

In this work, we propose a unified framework that views causal estimation as a risk minimization problem under uncertainty. We hypothesize that by explicitly modeling the uncertainty of the propensity score—using the bootstrap to estimate the distribution of \(e(X)\)—we can train outcome models jointly to be robust to propensity estimation errors. This approach moves beyond the binary "correct vs. incorrect" specification dichotomy and allows for a more flexible, distributionally robust optimization of the ATE.

Our contributions are as follows:
\begin{enumerate}
    \item We review the standard observational study framework and re-interpret Outcome Regression and IPWRA through the lens of ML distribution shift and importance sampling.
    \item We provide a theoretical unification, showing that the canonical Hajek IPW estimator is a special case of IPWRA using a dummy constant regressor.
    \item We define the "ATE Risk Function" and demonstrate that training outcome models with IPW weights is a sufficient but not necessary condition for risk minimization.
    \item We propose a robust training procedure that utilizes bootstrap-derived uncertainty quantification of the propensity score to optimize the outcome models jointly, effectively exploiting model degrees of freedom to cancel bias.
    \item We validate these insights through simulation studies, comparing our proposed robust framework against standard AIPW and IPWRA under varying degrees of model misspecification and data sparsity.
\end{enumerate}

By unifying classical causal assumptions with modern learning theory, this paper aims to provide new insights into why doubly robust methods succeed and how they can be further optimized for real-world, finite-sample applications.

\section{Framework and Assumptions}

We adopt the Neyman-Rubin potential outcomes framework \cite{rubin1974estimating}. We assume the observed data consists of \(n\) i.i.d. samples \(\mathcal{D} = \{(X_i, Z_i, Y_i)\}_{i=1}^n\), drawn from a superpopulation \(\mathcal{P}\). Here, \(X_i \in \mathcal{X} \subseteq \mathbb{R}^d\) denotes the pre-treatment covariates, \(Z_i \in \{0, 1\}\) denotes the binary treatment assignment, and \(Y_i \in \mathbb{R}\) denotes the observed outcome.

For each unit \(i\), we posit the existence of two potential outcomes: \(Y_i(1)\), the outcome had the unit been treated, and \(Y_i(0)\), the outcome had the unit been in the control group. The fundamental problem of causal inference is that we observe only one of these outcomes for each unit: \(Y_i = Z_i Y_i(1) + (1-Z_i) Y_i(0)\).

Our primary estimand is the ATE, defined as:
\begin{equation}
\tau = \mathbb{E}[Y(1) - Y(0)] = \mathbb{E}[Y(1)] - \mathbb{E}[Y(0)]
\end{equation}
To identify \(\tau\) from the observational distribution \(P(X, Z, Y)\), we require the following standard assumptions:

\begin{enumerate}
    \item \textbf{SUTVA (Stable Unit Treatment Value Assumption):}
    This assumption consists of two components:
    \begin{itemize}
        \item \textit{No Interference:} The potential outcomes of unit \(i\) are unaffected by the treatment assignment of unit \(j\) (i.e., \(Y_i(z_1, \dots, z_n) = Y_i(z_i)\)).
        \item \textit{Consistency:} The observed outcome \(Y\) equals the potential outcome corresponding to the actual treatment received. This implies that the treatment is well-defined and there are no hidden variations in treatment versions.
    \end{itemize}

    \item \textbf{Ignorability (Unconfoundedness):}
    Conditional on the observed covariates \(X\), the treatment assignment is independent of the potential outcomes:
    \begin{equation}
    \{Y(1), Y(0)\} \perp \!\!\! \perp Z \mid X
    \end{equation}
    This assumption implies that within strata defined by \(X\), the treatment assignment is effectively random. It allows us to equate the observational conditional mean \(\mathbb{E}[Y|X, Z=z]\) with the causal conditional mean \(\mathbb{E}[Y(z)|X]\).

    \item \textbf{Overlap (Positivity):}
    The propensity score \(e(X) = P(Z=1|X)\) is strictly bounded away from 0 and 1:
    \begin{equation}
    0 < e(X) < 1 \quad \forall X \in \mathcal{X}
    \end{equation}
    This ensures that for every value of \(X\) in the population, there is a non-zero probability of observing both treated and control units, preventing the need for extrapolation into regions of the covariate space where data for a specific treatment arm is non-existent.
\end{enumerate}

Under these assumptions, the ATE is non-parametrically identified by the Adjustment Formula:
\begin{equation}
\label{eq:identification}
\tau = \mathbb{E}_X \left[ \mathbb{E}[Y|X, Z=1] - \mathbb{E}[Y|X, Z=0] \right]
\end{equation}

Most estimation methods, including Outcome Regression and IPWRA, aim to estimate the conditional expectations in Equation \eqref{eq:identification} using finite samples.


\section{Outcome Regression and Distribution Shift}
\label{sec:OR}

Having established the identification assumptions, we now turn to estimating the ATE. The most direct approach, commonly referred to as Outcome Regression (OR), involves modeling the conditional expectation of the outcome given covariates and treatment.

While standard causal inference literature treats this as a missing data imputation problem, we analyze it here through the lens of Machine Learning. We argue that Outcome Regression is fundamentally a \textit{distribution shift} problem. The model is trained on one distribution (the treated or control sub-population) but must generalize to a different target distribution (the full population).

\subsection{The Procedure and the Shift}

The standard Outcome Regression procedure estimates the potential outcome means \(\mu_1(x) = \mathbb{E}[Y|X=x, Z=1]\) and \(\mu_0(x) = \mathbb{E}[Y|X=x, Z=0]\) using two separate regression models, \(\hat{f}_1\) and \(\hat{f}_0\).

The training objective for \(\hat{f}_1\) (and analogously for \(\hat{f}_0\)) is to minimize the empirical risk over the sub-population of treated units:
\begin{equation}
    \hat{f}_1 = \arg\min_{f \in \mathcal{F}} \frac{1}{n_1} \sum_{i: Z_i=1} \mathcal{L}(f(X_i), Y_i)
\end{equation}
where \(\mathcal{L}\) is a loss function (e.g., squared error). However, to estimate the ATE \(\tau\), we apply these models to the \textit{entire} sample:
\begin{equation}
    \hat{\tau} = \frac{1}{n} \sum_{i=1}^n (\hat{f}_1(X_i) - \hat{f}_0(X_i))
\end{equation}
This reveals the distribution shift. Let \(P(X)\) denote the marginal distribution of covariates in the full population.
\begin{itemize}
    \item \textbf{Training Distribution:} The model \(\hat{f}_1\) is trained on samples drawn from \(P(X | Z=1)\). By Bayes' rule, \(P(X | Z=1) \propto e(X) P(X)\). The training data is weighted by the propensity score.
    \item \textbf{Target Distribution:} The estimator \(\hat{\tau}\) requires the model to be accurate over \(P(X)\), the full population distribution.
\end{itemize}

If the propensity score \(e(X)\) varies across \(\mathcal{X}\)—which it must for there to be confounding—the training and testing distributions differ. A model that minimizes risk on the treated group is not guaranteed to minimize risk on the full population unless specific conditions are met.

\subsection{Correct Specification and Proper Scoring Rules}

To justify Outcome Regression despite this shift, we must assume \textit{correct specification}. We generalize this concept beyond Mean Squared Error (MSE) to the class of strictly proper scoring rules.

\begin{definition}[Strictly Proper Scoring Rule]
Let \(Y\) be a random variable with conditional distribution \(P(Y|X)\). A loss function \(\mathcal{L}(\hat{y}, y)\) is a \textbf{strictly proper scoring rule} for a statistical functional \(T(P)\) (e.g., the mean) if the expected loss is uniquely minimized by the true functional:
\begin{equation}
    T(P(\cdot|x)) = \arg\min_{z} \mathbb{E}_{Y|X=x} [\mathcal{L}(z, Y)]
\end{equation}
\end{definition}
For example, the Squared Error is a strictly proper scoring rule for the conditional mean, while the Cross-Entropy loss is strictly proper for the conditional probability.

\begin{definition}[Correct Specification]
A hypothesis class \(\mathcal{F}\) is \textbf{correctly specified} with respect to the potential outcome \(Y(z)\) if the true conditional outcome function \(\mu^*_z(x) = \mathbb{E}[Y(z)|X=x]\) satisfies \(\mu^*_z \in \mathcal{F}\).
\end{definition}

This definition implies that the "truth" is attainable: there exists a function in our model class that achieves the Bayes optimal error (the irreducible noise) at every point \(x \in \mathcal{X}\).

\subsection{Equivalence Under Correct Specification}

We now prove that under the assumptions of correct specification and overlap, the distribution shift inherent in Outcome Regression becomes irrelevant. The optimizer on the biased training distribution is identical to the optimizer on the target population distribution.

\begin{theorem}[Invariance to Distribution Shift]
Assume the Overlap assumption holds (\(e(x) > 0 \ \forall x\)). Let \(\mathcal{L}\) be a strictly proper scoring rule for the conditional mean. If the model class \(\mathcal{F}\) is correctly specified (i.e., \(\mu^*_1 \in \mathcal{F}\)), then the minimizer of the population risk \(R_{pop}\) and the conditional treated risk \(R_{treated}\) are identical:
\begin{equation}
    \arg\min_{f \in \mathcal{F}} \mathbb{E}_{X \sim P(X)} \left[ \mathbb{E}_{Y|X, Z=1}[\mathcal{L}(f(X), Y)] \right] = \arg\min_{f \in \mathcal{F}} \mathbb{E}_{X \sim P(X|Z=1)} \left[ \mathbb{E}_{Y|X, Z=1}[\mathcal{L}(f(X), Y)] \right]
\end{equation}
\end{theorem}

\begin{proof}
Let \( r(f, x) = \mathbb{E}_{Y|X=x, Z=1} [\mathcal{L}(f(x), Y)] \) be the pointwise risk at a specific covariate value \( x \). 
By the Ignorability assumption, \( Y \perp Z \mid X \), which implies \( P(Y|X=x, Z=1) = P(Y(1)|X=x) \). Thus, minimizing the risk of the observed treated outcome is equivalent to minimizing the risk of the potential outcome \( Y(1) \).

Since \( \mathcal{L} \) is a strictly proper scoring rule, the expected loss \( r(f, x) \) is minimized uniquely by the true conditional functional \( \mu^*_1(x) \).
Because the hypothesis class \( \mathcal{F} \) is correctly specified (i.e., \( \mu^*_1 \in \mathcal{F} \)), the function \( \mu^*_1 \) achieves the global minimum of the risk pointwise for all \( x \in \mathcal{X} \):
\begin{equation}
    r(\mu^*_1, x) \leq r(f, x) \quad \forall f \in \mathcal{F}, \forall x \in \mathcal{X}
\end{equation}

Now consider the aggregate risks over the population and treated distributions.
For the population risk:
\begin{equation}
    R_{pop}(f) = \mathbb{E}_{X \sim P(X)} [r(f, X)] = \int_{\mathcal{X}} r(f, x) dP(x)
\end{equation}
Since \( r(f, x) \) is minimized by \( \mu^*_1(x) \) for all \( x \), the integral is minimized by \( f = \mu^*_1 \).

For the treated (training) risk:
\begin{equation}
    R_{treated}(f) = \mathbb{E}_{X \sim P(X|Z=1)} [r(f, X)] = \int_{\mathcal{X}} r(f, x) \underbrace{\frac{e(x)}{P(Z=1)} dP(x)}_{dP(X|Z=1)}
\end{equation}
The Overlap assumption ensures that \( e(x) > 0 \) for all \( x \in \mathcal{X} \). Therefore, the weighting factor \( \frac{e(x)}{P(Z=1)} \) is strictly positive everywhere on the support of \( P(X) \).
Minimizing a weighted sum of non-negative loss gaps \( r(f, x) - r(\mu^*_1, x) \) with strictly positive weights yields the same solution as the unweighted minimization. Any function \( f \) that deviates from \( \mu^*_1 \) on a set of non-zero measure will increase the risk in both integrals.

Thus, under correct specification and overlap, the minimizer for both objectives is identical: \( f^* = \mu^*_1 \).
\end{proof}

The same argument holds for the control group, assuming \(e(x) < 1\).

This result provides the theoretical justification for Outcome Regression. If our model class is rich enough to contain the truth (e.g., a neural network with infinite capacity or a correctly specified parametric model) and we have infinite data, the implicit weighting by \(e(X)\) in the training set does not bias the solution. The model fits the conditional expectation perfectly everywhere, rendering the distribution shift harmless.

However, in practice, models are rarely correctly specified, and data is finite. When \(\mu^*_1 \notin \mathcal{F}\) (misspecification), the minimizers of \(R_{treated}\) and \(R_{pop}\) diverge. The standard Outcome Regression will prioritize fitting regions with high propensity scores (high density in training) at the expense of regions with low propensity scores (low density in training but potentially high density in the population). This creates a bias in the ATE estimate, necessitating the robust methods we discuss next.


\section{IPWRA as Importance Weighted Risk Minimization}
\label{sec:IPWRAimportance}

In the previous section, we established that standard Outcome Regression minimizes risk on the treated sub-population \(P(X|Z=1)\). While this yields the correct population minimizer under correct specification, it fails under model misspecification due to the distribution shift induced by the propensity score \(e(X)\). We now introduce Inverse Probability Weighted Regression Adjustment (IPWRA) not merely as a heuristic combination of methods, but as a principled solution to this covariate shift problem via importance sampling.

\subsection{Correcting the Covariate Shift}

To estimate the ATE, we desire a model \(\mu_1\) that minimizes the risk over the target population distribution \(P(X)\). However, we only observe outcomes for the treated units, distributed as \(P(X|Z=1)\).
Recall that the density ratio between the target and training distributions is:
\begin{equation}
    w(x) = \frac{P(X)}{P(X|Z=1)} \propto \frac{P(X)}{e(X)P(X)} \propto \frac{1}{e(X)}
\end{equation}
where the overlap assumption guarantees \(e(X) > 0\).

IPWRA explicitly corrects this shift by weighting the ERM objective. Instead of standard MSE, we solve:
\begin{equation}
    \hat{\theta}_1 = \arg\min_{\theta} \sum_{i: Z_i=1} \frac{1}{\hat{e}(X_i)} \left( Y_i - \mu(X_i; \theta) \right)^2
\end{equation}
By weighting each training point by the inverse of its probability of selection, we construct a pseudo-population that mimics the covariate distribution of the full population. Consequently, the learning algorithm "sees" the data as if it were drawn from \(P(X)\), forcing the model to prioritize errors in regions of the covariate space that are rare in the treatment group but common in the general population.

\subsection{Double Robustness via Bias Elimination}

The power of IPWRA lies in its double robustness: the ATE estimate remains consistent if \textit{either} the propensity model \(\hat{e}(X)\) \textit{or} the outcome model \(\hat{\mu}(X)\) is correctly specified. We provide a proof of this property, specifically highlighting the role of the intercept term in absorbing bias.

\begin{theorem}[Double Robustness of IPWRA]
\label{thm:DRIPWRA}
Let \(\hat{\tau}_{IPWRA} = \frac{1}{n} \sum_{i=1}^n (\hat{\mu}_1(X_i) - \hat{\mu}_0(X_i))\) be the ATE estimator, where \(\hat{\mu}_z\) is trained via weighted MSE with weights \(W_i = \frac{1}{\hat{e}(X_i)}\) (for \(Z=1\)) and a free intercept term. The estimator is consistent for \(\tau\) if:
\begin{enumerate}
    \item[(a)] The propensity score model is correct (i.e., \(\hat{e}(X) \to e(X)\)), even if the outcome model is misspecified; OR
    \item[(b)] The outcome model is correctly specified (i.e., \(\mu^*_z \in \mathcal{F}\)), even if the propensity weights are incorrect.
\end{enumerate}
\end{theorem}

\begin{proof}
We focus on the identification of \(\mathbb{E}[Y(1)]\); the proof for \(\mathbb{E}[Y(0)]\) is symmetric.

\textbf{Case (a): Correct Propensity, Misspecified Outcome.}
Assume \(\hat{e}(X) = e(X)\) (or converges to it). The outcome model \(\mu(X; \theta)\) is trained by minimizing the weighted squared error. Let the model take the form \(\mu(X; \theta) = \beta_0 + f(X; \beta_{rest})\), where \(\beta_0\) is a free intercept.
The First Order Condition for the intercept \(\beta_0\) in the weighted least squares optimization requires that the derivative of the loss with respect to \(\beta_0\) is zero:
\begin{equation}
    \frac{\partial \mathcal{L}}{\partial \beta_0} = -2 \sum_{i: Z_i=1} \frac{1}{e(X_i)} (Y_i - \hat{\mu}_1(X_i)) = 0
\end{equation}
In the infinite sample limit, this empirical sum converges to the expectation over the treated distribution:
\begin{equation}
    \mathbb{E}_{X|Z=1} \left[ \frac{1}{e(X)} (Y(1) - \hat{\mu}_1(X)) \right] = 0
\end{equation}
Expanding the expectation with respect to the full population density \(P(X)\):
\begin{equation}
    \int \frac{1}{e(x)} (\mathbb{E}[Y(1)|x] - \hat{\mu}_1(x)) \underbrace{e(x) P(x)}_{P(x|Z=1)} dx = 0
\end{equation}
The propensity scores \(e(x)\) cancel out, yielding:
\begin{equation}
    \int (\mathbb{E}[Y(1)|x] - \hat{\mu}_1(x)) P(x) dx = 0 \implies \mathbb{E}_{pop}[\hat{\mu}_1(X) - Y(1)] = 0
\end{equation}
This implies that the population average of the predicted values equals the population average of the true potential outcomes. The "free intercept" acts as a bias absorption term. Even if the shape of \(\hat{\mu}_1(X)\) is completely wrong (misspecified), the weighting ensures the \textit{average} bias is zero. Thus, \(\hat{\tau}\) is consistent.

\textbf{Case (b): Correct Outcome, Incorrect Propensity.}
Assume \(\mu^*_1 \in \mathcal{F}\) (correct specification). In this case, we rely on the result from Section \ref{sec:OR}. Minimizing the weighted risk:
\begin{equation}
    \arg\min_{f} \mathbb{E}_{pop} \left[ \frac{e(X)}{\hat{e}(X)} (Y(1) - f(X))^2 \right]
\end{equation}
is equivalent to minimizing the unweighted risk, provided the weights are non-zero (Overlap). Since the model is correctly specified, the minimum of the loss is achieved at the true conditional mean \(\mu^*_1(X)\), where the irreducible error is minimized pointwise. The specific values of the weights \(\frac{1}{\hat{e}(X)}\) become irrelevant to the solution \(\hat{\mu}_1\), as the "true" function minimizes the loss at every point \(x\).
\end{proof}

\subsection{Hajek Estimator as Restricted IPWRA}

The connection between regression adjustment and propensity weighting is often presented as a choice between two distinct methodologies. We argue instead that they are members of the same family. Specifically, the canonical Hajek IPW estimator is mathematically equivalent to the IPWRA estimator restricted to a hypothesis class of constant functions.

Consider the IPWRA optimization problem for the treated group where the regression model is restricted to a simple intercept \(\mu(X) = \beta_0\):
\begin{equation}
    \hat{\beta}_0 = \arg\min_{\beta} \sum_{i: Z_i=1} \frac{1}{\hat{e}(X_i)} (Y_i - \beta)^2
\end{equation}
The closed-form solution to this weighted least squares problem is the weighted mean of the outcomes:
\begin{equation}
    \hat{\beta}_0 = \frac{\sum_{i=1}^n \frac{Z_i Y_i}{\hat{e}(X_i)}}{\sum_{i=1}^n \frac{Z_i}{\hat{e}(X_i)}}
\end{equation}
This is exactly the definition of the Hajek estimator for \(\mathbb{E}[Y(1)]\). The same logic applies to the control group. Consequently, the Hajek ATE estimator can be viewed as an IPWRA estimator with a "dummy" model that assumes homogeneous treatment effects (i.e., \(\mu_1(X) = \mu_1\) and \(\mu_0(X) = \mu_0\)).

This perspective highlights the trade-off inherent in model selection. The Hajek estimator is "safe" in that it has no complex outcome model to misspecify, but it is "rigid" because it cannot exploit covariate information to reduce variance. Full IPWRA relaxes this restriction, allowing \(\mu(X)\) to vary with \(X\), thereby potentially reducing the residual variance \((Y - \mu(X))^2\) and improving the efficiency of the estimate, provided the outcome model is not grossly misspecified.

In summary, IPWRA provides a safety net: it recovers the population ATE by strictly matching the data distribution via weighting (Case a) or by relying on the model's capacity to find the truth regardless of the distribution (Case b). This duality motivates our later discussion on optimizing the "Risk Function," where we explore if we can exploit the degrees of freedom in the model to achieve robustness even when \textit{both} components are imperfect.


\section{Beyond Individual Identification: The ATE Risk Function}
\label{sec:ATErisk}

So far, we have motivated estimators like Outcome Regression and IPWRA based on their ability to correctly identify the potential outcome means \(\mathbb{E}[Y(1)]\) and \(\mathbb{E}[Y(0)]\) individually. The implicit philosophy is that to estimate the difference \(\tau\), one must first perfectly estimate the components. However, this component-wise approach is akin to solving a harder problem than necessary. In machine learning terms, standard causal estimators optimize a loss function that upper-bounds the true objective. In this section, we define the true "ATE Risk Function" and demonstrate that the conditions imposed by IPWRA and AIPW are sufficient, but not necessary, to minimize it.

\subsection{The Bias Cancellation Condition}

The ultimate goal is to minimize the error of the ATE estimator \(\hat{\tau}\). In the limit of infinite data (or considering the consistency of the estimator), we are concerned with the Squared Bias of the ATE. We define the \textbf{ATE Risk Function} \(\mathcal{R}_{ATE}\) as:
\begin{equation}
\label{eq:ate_risk}
\mathcal{R}_{ATE}(\mu_1, \mu_0) = \left( \mathbb{E}_{pop}[\mu_1(X) - \mu_0(X)] - \mathbb{E}_{pop}[Y(1) - Y(0)] \right)^2
\end{equation}
By linearity of expectation, we can decompose this risk into the bias of the treated model and the bias of the control model. Let \(B_1(\mu_1) = \mathbb{E}_{pop}[\mu_1(X) - Y(1)]\) and \(B_0(\mu_0) = \mathbb{E}_{pop}[\mu_0(X) - Y(0)]\). The risk becomes:
\begin{equation}
\mathcal{R}_{ATE} = (B_1(\mu_1) - B_0(\mu_0))^2
\end{equation}
Standard approaches (OR, IPWRA) aim to achieve \(B_1 = 0\) and \(B_0 = 0\). However, it is immediately obvious that \(\mathcal{R}_{ATE} = 0\) whenever:
\begin{equation}
    B_1(\mu_1) = B_0(\mu_0)
\end{equation}
This \textit{Bias Cancellation Condition} implies that we do not need unbiased models of the potential outcomes; we only need models whose biases are equal in expectation.

\subsection{Sufficiency of Standard Estimators}

We now prove that the optimization objectives of IPWRA and AIPW (the Canonical Doubly Robust Estimator) essentially enforce the stricter, sufficient condition \(B_1 = B_0 = 0\), rather than the necessary condition \(B_1 = B_0\).

\begin{theorem}[Sufficiency of IPWRA]
Assume we have true propensity scores \(e(X)\). If the outcome models \(\mu_1\) and \(\mu_0\) are trained via IPWRA (weighted MSE with free intercept), then \(B_1(\mu_1) = 0\) and \(B_0(\mu_0) = 0\), implying \(\mathcal{R}_{ATE} = 0\). This condition is sufficient but not necessary.
\end{theorem}

\begin{proof}
As shown in Section \ref{sec:IPWRAimportance}, the First Order Condition of the weighted least squares objective with a free intercept ensures that the weighted residuals sum to zero. In the population limit:
\begin{equation}
    \mathbb{E}_{X \sim P(X|Z=1)} \left[ \frac{1}{e(X)} (Y(1) - \mu_1(X)) \right] = 0
\end{equation}
Correcting for the distribution shift yields \(\mathbb{E}_{pop}[Y(1) - \mu_1(X)] = 0\), so \(B_1(\mu_1) = 0\). Similarly, \(B_0(\mu_0) = 0\). 
The risk is \((0 - 0)^2 = 0\). 
To see why this is not necessary, consider a simpler model pair \(\tilde{\mu}_1(X) = \mu_1(X) + f(X)\) and \(\tilde{\mu}_0(X) = \mu_0(X) + f(X)\) for some function \(f(X) \neq 0\). These models are misspecified, yet \(\mathcal{R}_{ATE}(\tilde{\mu}_1, \tilde{\mu}_0) = (f(X) - f(X))^2 = 0\). IPWRA disallows this solution, forcing \(f(X)=0\).
\end{proof}

\begin{theorem}[Sufficiency of AIPW]
The Augmented IPW estimator constructs an estimate \(\hat{\tau}_{AIPW} = \hat{\mu}_{1, DR} - \hat{\mu}_{0, DR}\), where \(\hat{\mu}_{z, DR}\) is the model prediction corrected by the inverse-weighted residual. If the propensity scores are correct, \(\mathbb{E}[\hat{\mu}_{1, DR}] = \mathbb{E}[Y(1)]\) and \(\mathbb{E}[\hat{\mu}_{0, DR}] = \mathbb{E}[Y(0)]\), minimizing the risk via individual unbiasedness.
\end{theorem}

\begin{proof}
The AIPW estimator for the treated mean is:
\begin{equation}
    \hat{\mu}_{1, DR} = \mathbb{E}_{pop}[\mu_1(X)] + \mathbb{E}_{pop}\left[ \frac{Z(Y - \mu_1(X))}{e(X)} \right]
\end{equation}
Under the assumption of correct propensity scores and ignorability, the expectation of the correction term is:
\begin{equation}
    \mathbb{E}\left[ \frac{Z}{e(X)}(Y(1) - \mu_1(X)) \right] = \mathbb{E}\left[ \frac{e(X)}{e(X)}(Y(1) - \mu_1(X)) \right] = \mathbb{E}[Y(1)] - \mathbb{E}[\mu_1(X)]
\end{equation}
Substituting this back:
\begin{equation}
    \mathbb{E}[\hat{\mu}_{1, DR}] = \mathbb{E}[\mu_1(X)] + (\mathbb{E}[Y(1)] - \mathbb{E}[\mu_1(X)]) = \mathbb{E}[Y(1)]
\end{equation}
Thus, the AIPW mechanism explicitly calculates the bias \(B_1\) and subtracts it, enforcing the result to be unbiased individually. Like IPWRA, it effectively targets \(B_1=0\) and \(B_0=0\), satisfying the sufficient condition for risk minimization.
\end{proof}

In conclusion, standard doubly robust estimators are "over-constrained." They utilize the propensity score information to purge bias from \(\mu_1\) and \(\mu_0\) independently. While safe, this strategy ignores the degrees of freedom available in the joint space of \((\mu_1, \mu_0)\). Recognizing that we only need \(B_1 = B_0\) opens the door for new optimization strategies: we can potentially tolerate large individual model biases—which might allow for lower variance or simpler models—as long as we can ensure those biases are structurally aligned.


\section{Optimizing for Robustness: Exploiting Model Freedom}
\label{sec:robustness}

The revelation that standard doubly robust estimators solve a harder problem than necessary naturally raises the question: can we do better? If the condition \(B_1 = B_0\) is easier to satisfy than \(B_1 = B_0 = 0\), should we not simply minimize the ATE Risk Function directly? In this section, we argue that while direct minimization under a single fixed propensity score is ill-posed, this surplus of model capacity provides a unique opportunity to robustify our estimates against the greatest source of instability in causal inference: the uncertainty of the propensity score itself.

\subsection{The Indeterminacy of Direct Minimization}

Consider the task of finding functions \(\mu_1\) and \(\mu_0\) to minimize \(\mathcal{R}_{ATE}\) given a single, estimated propensity function \(\hat{e}(X)\). The objective is to ensure:
\begin{equation}
    \mathbb{E}_{pop}\left[\frac{Z(Y - \mu_1(X))}{\hat{e}(X)}\right] \approx \mathbb{E}_{pop}\left[\frac{(1-Z)(Y - \mu_0(X))}{1-\hat{e}(X)}\right]
\end{equation}
This problem is heavily underdetermined. As noted in Section \ref{sec:ATErisk}, infinitely many pairs of functions satisfy this condition in expectation. In fact, the trivial solution where \(\mu_1(X) = \hat{\beta}_1\) and \(\mu_0(X) = \hat{\beta}_0\) (constants) is sufficient to drive the empirical risk to zero, provided we calculate them via weighted means. (This recovers the Hajek estimator.)

Consequently, simply "relaxing" the constraints of IPWRA to target the ATE risk directly with a fixed \(\hat{e}(X)\) does not necessarily yield a better estimator; it merely yields a non-unique one. Without further regularization or objectives, the optimization landscape is flat, and the "optimal" solution is quite arbitrary. However, this indeterminacy changes the moment we acknowledge that \(\hat{e}(X)\) is not a fixed truth, but a random variable with estimation error.

\subsection{Robustness via Joint Optimization}

In observational studies, \(\hat{e}(X)\) is almost always misspecified or subject to finite-sample variance. Standard methods crumble under this uncertainty: if \(\hat{e}(X)\) is wrong, the weights \(1/\hat{e}(X)\) are wrong, and the enforced condition \(B_1(\mu_1) = 0\) fails to hold (leading to bias in IPWRA).

However, the "degrees of freedom" identified in Section \ref{sec:ATErisk} allow us to construct a defense. Since we do not strictly need \(B_1=0\) and \(B_0=0\), we can seek a pair \((\mu_1, \mu_0)\) that maintains the cancellation condition \(B_1(\mu_1) \approx B_0(\mu_0)\) \textit{across a distribution of plausible propensity scores}.

Let \(\mathcal{E}\) be a belief distribution over the true propensity score function (e.g., a posterior from a Bayesian model or a bootstrap distribution). We propose minimizing the expected ATE Risk over this uncertainty:
\begin{equation}
    \min_{\mu_1, \mu_0} \mathbb{E}_{e \sim \mathcal{E}} \left[ \mathcal{R}_{ATE}(\mu_1, \mu_0; e) \right]
\end{equation}
Substituting the bias decomposition, this is equivalent to minimizing:
\begin{equation}
    \min_{\mu_1, \mu_0} \mathbb{E}_{e \sim \mathcal{E}} \left[ \left( \mathbb{E}_{pop}\left[\frac{Z(Y - \mu_1)}{e(X)}\right] - \mathbb{E}_{pop}\left[\frac{(1-Z)(Y - \mu_0)}{1-e(X)}\right] \right)^2 \right]
\end{equation}
This objective fundamentally transforms the role of the outcome model. Instead of fitting the outcome \(Y\) perfectly (to minimize variance) or eliminating bias for a single \(\hat{e}\) (standard IPWRA), the outcome models \(\mu_1\) and \(\mu_0\) are now trained to be \textit{structurally coupled}. The optimization effectively searches for functions where the error induced by a perturbation in \(e(X)\) on the treated side is counterbalanced by the error induced on the control side.

By exploiting the freedom to be biased—as long as that bias is symmetric—we can find solutions that are stable even when the propensity score is highly uncertain. The "wasted capacity" of standard methods is thus repurposed to purchase robustness.


\section{A Robust Estimation Algorithm}
\label{sec:alg}

Having established the theoretical motivation for minimizing the expected ATE Risk over a distribution of propensity scores, we now propose a concrete algorithm to achieve this. The core idea is to replace the standard "point estimate" workflow of causal inference with a "distributionally robust" workflow. Instead of conditioning on a single estimated propensity vector \(\hat{e}\), we generate a collection of plausible propensity environments and solve for the outcome models that minimize the aggregate bias across these environments.

\subsection{The General Algorithm}

Let \(\mathcal{E}\) represent a distribution of propensity score functions that captures our epistemic uncertainty about the treatment assignment mechanism. This distribution could be derived from Bayesian posteriors, ensemble methods, or other uncertainty quantification techniques. The goal is to find outcome model parameters \(\theta_1, \theta_0\) that minimize the expected squared difference in bias.

The general procedure is as follows:

\begin{enumerate}
    \item \textbf{Uncertainty Quantification:}
    Sample \(B\) propensity score functions \(\{e^{(1)}, e^{(2)}, \dots, e^{(B)}\}\) from the distribution \(\mathcal{E}\). Each sample \(e^{(b)}(X)\) represents a plausible "world" of treatment assignment probabilities consistent with the observed data.

    \item \textbf{Construct Robust Loss:}
    Define the empirical ATE Risk across these \(B\) worlds. For a given set of parameters \((\theta_1, \theta_0)\), the loss is the average squared difference between the weighted residuals of the treated and control groups:
    \begin{equation}
        \mathcal{L}_{robust}(\theta_1, \theta_0) = \frac{1}{B} \sum_{b=1}^B \left( \hat{B}_1^{(b)}(\theta_1) - \hat{B}_0^{(b)}(\theta_0) \right)^2
    \end{equation}
    where \(\hat{B}_z^{(b)}\) is the estimated bias in world \(b\), calculated via importance sampling:
    \begin{equation}
        \hat{B}_1^{(b)}(\theta_1) = \frac{1}{N} \sum_{i: Z_i=1} \frac{1}{e^{(b)}(X_i)} (Y_i - \mu_1(X_i; \theta_1))
    \end{equation}
    \begin{equation}
        \hat{B}_0^{(b)}(\theta_0) = \frac{1}{N} \sum_{i: Z_i=0} \frac{1}{1 - e^{(b)}(X_i)} (Y_i - \mu_0(X_i; \theta_0))
    \end{equation}

    \item \textbf{Joint Optimization:}
    Minimize \(\mathcal{L}_{robust}\) with respect to \(\theta_1\) and \(\theta_0\) simultaneously. Unlike standard methods that solve for \(\theta_1\) and \(\theta_0\) in isolation, this step forces the optimizer to find a solution where the biases are structurally correlated. If a parameter update reduces bias in world \(b\) for the treated group, the joint loss encourages a corresponding move in the control group to maintain the cancellation condition \(B_1 \approx B_0\).

    \item \textbf{Estimation:}
    Once the optimal parameters \(\hat{\theta}_1, \hat{\theta}_0\) are found, compute the ATE on the full sample:
    \begin{equation}
        \hat{\tau} = \frac{1}{N} \sum_{i=1}^N \left( \mu_1(X_i; \hat{\theta}_1) - \mu_0(X_i; \hat{\theta}_0) \right)
    \end{equation}
\end{enumerate}

\subsection{A Bootstrap Realization}

For the purpose of demonstration in this project, we implement a specific realization of this framework where the uncertainty distribution \(\mathcal{E}\) is approximated using the nonparametric bootstrap.

\textbf{Bootstrap as a Proxy for Uncertainty:}
We generate the collection of propensity scores \(\{e^{(b)}\}_{b=1}^B\) by resampling the original dataset \(N\) times with replacement and fitting a logistic regression model to each resampled dataset. We then use these \(B\) models to predict the propensity scores for the \textit{original} units.

It is important to note that this is a "naive" realization of the robust framework. The bootstrap primarily captures the variance of the parameter estimates due to sampling noise (finite sample bias). It does not explicitly account for \textit{model misspecification} of the propensity score itself (e.g., if the true propensity is nonlinear but we fit a linear logistic regression). A more advanced implementation might use Bayesian non-parametric models (like Gaussian Processes or BART) or ensembles of varied model classes to capture deeper epistemic uncertainty.

However, even this simple bootstrap realization is sufficient to demonstrate the core mechanism: by exposing the outcome model to the "vibration" of the propensity weights, we prevent it from overfitting to a single, potentially noisy weight vector. The joint optimization finds a "sweet spot" in the parameter space that remains stable across the bootstrap distribution, thereby minimizing the ATE Risk more effectively than standard point-estimation methods.


\section{Simulation Study}

To validate the theoretical insights discussed in Section \ref{sec:alg}, we conducted a demonstrative simulation study comparing the standard IPWRA estimator against our proposed Joint Robust Estimator (JRE). The primary goal of this experiment is not to claim universal superiority, but to illustrate the specific mechanism of our method: utilizing the estimated uncertainty of the propensity score to mitigate risk when the outcome model is misspecified and propensity estimation is subject to uncertainty.

\subsection{Data Generating Process}

We designed a data-generating process where the propensity score model is correctly specified (allowing for valid bootstrap uncertainty estimation), but the outcome model varies in its degree of misspecification. We generated \(N\) independent and identically distributed units with covariates \(X \sim \mathcal{N}(0, I_5)\).

\textbf{Treatment Assignment:}
The treatment \(Z\) was assigned according to a logistic model:
\begin{equation}
    P(Z=1|X) = \sigma(0.5 X_1 - 0.5 X_2)
\end{equation}
This ensures that the propensity score is identifiable and that standard logistic regression (used in both IPWRA and JRE) is the correct model class.

\textbf{Outcome Model and Misspecification:}
We introduced a mixing parameter \(t \in [0, 1]\) to control the severity of misspecification. We defined a "correct" linear outcome surface (\(Y^c\)) and a "misspecified" nonlinear surface (\(Y^m\)):
\begin{align*}
    Y^c(0) &= X_1 + X_2 + \epsilon \\
    Y^m(0) &= X_1 + X_2 + 0.5(X_1^2 + X_2^2) + \epsilon \\
    Y(0) &= (1-t)Y^c(0) + t Y^m(0)
\end{align*}
The potential outcome under treatment, \(Y(1)\), was constructed to introduce heterogeneous treatment effects dependent on the misspecification level \(t\):
\begin{equation}
    Y(1) = Y(0) + 2 + t(0.5 X_3 + 0.2 X_1 X_2)
\end{equation}
Consequently, when \(t=0\), the treatment effect is a constant \(\tau=2\) (Homogeneous). When \(t > 0\), the treatment effect becomes Heterogeneous (\(\tau(X) = 2 + 0.5 X_3 + 0.2 X_1 X_2\)), introducing interaction terms that the linear estimators cannot perfectly capture.

\subsection{Results and Discussion}

We compared the MSE of the ATE estimates over 200 replications for sample sizes \(N \in \{100, 300, 500, 1000\}\). The Joint Robust Estimator used \(B=1000\) bootstrap samples to approximate the propensity uncertainty distribution \(\Pi(e)\).

Table \ref{tab:sim_results} summarizes the performance reduction (improvement) of the JRE compared to the standard IPWRA baseline.

\begin{table}[t]
\centering
\caption{Comparison of IPWRA and Joint Robust Estimator (JRE). Values represent the MSE of the ATE estimate. The "Reduction" column indicates the percentage decrease in MSE achieved by JRE relative to IPWRA.}
\label{tab:sim_results}
\begin{tabular}{llccc}
\hline
\textbf{Sample Size} (\(N\)) & \textbf{Misspecification} (\(t\)) & \textbf{IPWRA MSE} & \textbf{JRE MSE} & \textbf{Reduction (\%)} \\ \hline
\multirow{3}{*}{100}  & 0 (None) & 0.0554 & 0.0569 & -2.64\% \\
                      & 0.5 (Mild) & 0.0338 & 0.0313 & \textbf{7.44\%} \\
                      & 1 (Severe) & 0.1022 & 0.0977 & \textbf{4.47\%} \\ \hline
\multirow{3}{*}{300}  & 0 (None) & 0.0165 & 0.0165 & -0.30\% \\
                      & 0.5 (Mild) & 0.0129 & 0.0119 & \textbf{7.48\%} \\
                      & 1 (Severe) & 0.0330 & 0.0284 & \textbf{13.92\%} \\ \hline
\multirow{3}{*}{500}  & 0 (None) & 0.0077 & 0.0078 & -1.28\% \\
                      & 0.5 (Mild) & 0.0082 & 0.0074 & \textbf{9.87\%} \\
                      & 1 (Severe) & 0.0214 & 0.0189 & \textbf{11.55\%} \\ \hline
\multirow{3}{*}{1000} & 0 (None) & 0.0043 & 0.0043 & 0.03\% \\
                      & 0.5 (Mild) & 0.0035 & 0.0031 & \textbf{9.87\%} \\
                      & 1 (Severe) & 0.0101 & 0.0086 & \textbf{14.88\%} \\ \hline
\end{tabular}
\end{table}

The results illuminate several key behaviors of the proposed method:

\begin{enumerate}
    \item \textbf{Cost of Robustness (\(t=0\)):} When the outcome model is correctly specified (\(t=0\)), standard IPWRA is theoretically optimal among this class of estimators. In this regime, our JRE method performs slightly worse (negative reduction) or equivalently. This is expected: the "degrees of freedom" we exploit for robustness are unnecessary here, and the additional optimization over the bootstrap distribution effectively introduces noise into the estimation process.
    
    \item \textbf{Benefit of Robustness (\(t > 0\)):} As misspecification increases, the advantage of the JRE becomes clear. At \(t=1\), where the linear model fails to capture the quadratic nature of the data, JRE reduces the MSE by approximately 11-15\% across larger sample sizes. This confirms our hypothesis: by training the outcome models to minimize risk over the distribution of propensity scores, the JRE finds a solution where the biases of \(\hat{\mu}_1\) and \(\hat{\mu}_0\) are structurally correlated, allowing them to cancel out in the final ATE calculation.
    
    \item \textbf{Finite Sample Behavior:} Interestingly, the method shows significant gains even at \(N=1000\). Since the bootstrap quantifies \textit{finite sample} uncertainty of the propensity score estimation, this suggests that even with moderate data, the variability of \(\hat{e}(X)\) is large enough to destabilize standard IPWRA, whereas JRE successfully adapts to it.
\end{enumerate}

In summary, this simulation demonstrates that when we suspect the outcome or propensity score model might be too simple for the complex reality (a common scenario in observational studies), optimizing for the ATE risk function via bootstrap uncertainty provides a tangible safety buffer against estimation error.


\section{Conclusion}

In this work, we have re-examined the fundamental challenge of causal inference—estimating the Average Treatment Effect (ATE) from observational data—through the lens of machine learning and distribution shift adaptation. By stepping away from the traditional view of "missing data imputation," we demonstrated that standard Outcome Regression is fundamentally an Empirical Risk Minimization (ERM) problem under covariate shift. Within this framework, we showed that the Inverse Probability Weighted Regression Adjustment (IPWRA) is not merely a heuristic adjustment, but a principled importance-sampling correction that aligns the training distribution (treated/controlled units) with the target distribution (population).

Our theoretical analysis revealed that standard doubly robust estimators, including the canonical AIPW and IPWRA, enforce a \textit{sufficient but not necessary} condition for ATE consistency: they aim to eliminate the bias of the potential outcome models individually. We defined the true "ATE Risk Function," \(\mathcal{R}_{ATE}\), and showed that it is minimized whenever the biases of the treated and control models cancel out (\(B_1 = B_0\)). This insight exposed a significant "wasted capacity" in standard methods, which we exploited to propose the Joint Robust Estimator (JRE). By training outcome models to minimize the expected ATE risk over a bootstrap distribution of propensity scores, the JRE utilizes the available degrees of freedom to achieve robustness against propensity estimation uncertainty, a capability we validated through simulation studies.

\subsection{Discussion and Future Directions}

We conclude by highlighting five key implications of our framework that bridge the gap between causal inference and representation learning:

\textbf{1. Unification via Distribution Shift.}
We successfully unified a disparate family of estimators—Outcome Regression, Hajek IPW, and IPWRA—under the single framework of distribution shift adaptation. We demonstrated that the Hajek estimator is simply a special case of IPWRA restricted to a constant hypothesis class. This perspective demystifies "double robustness," framing it as a result of using a flexible model class (Outcome Regression) \textit{or} perfect importance weights (Propensity Score) to minimize the population risk.

\textbf{2. The Efficacy of Joint Robustness.}
Our simulation results confirm that the Joint Robust Estimator (JRE) outperforms standard IPWRA in regimes of model misspecification and finite-sample uncertainty. By explicitly optimizing for bias cancellation rather than individual accuracy, the JRE effectively "hedges" its bets, finding a solution that is stable across the likely range of propensity values.

\textbf{3. The Need for Robust Propensity Estimation.}
Our framework relies heavily on the quality of the uncertainty quantification for the propensity score. A point estimate \(\hat{e}(X)\) is insufficient for robust inference; one must capture the distribution of likely propensity functions. Consequently, we advocate for the use of robust estimators—such as the bootstrap, Bayesian logistic regression, or ensemble methods—to generate the support for the robust optimization objective.

\textbf{4. Rethinking the Propensity Loss Function.}
While we focused on robustifying the \textit{outcome} model against propensity errors, we largely treated the propensity estimation itself as a fixed upstream task trained via Logistic Loss (Maximum Likelihood). However, as argued in the literature on Covariate Balancing Propensity Scores (CBPS) \cite{imai2014covariate}, the Logistic Loss is suboptimal for causal inference. It optimizes for the predictive accuracy of treatment assignment (\(Z|X\)), whereas the ATE estimator depends on the \textit{balancing property} of the scores (moment conditions). A propensity model could have high predictive likelihood but yield weights that explode the variance of the ATE. Future iterations of this framework should replace the Logistic Loss with objectives that directly penalize covariate imbalance, ensuring that the "center" of our uncertainty set \(\Pi(e)\) is already aligned with the estimation goal.

\textbf{5. Toward End-to-End Causal Learning.}
Finally, our work reveals a lingering inefficiency in the causal modeling pipeline. Despite merging the training of \(\mu_1\) and \(\mu_0\), we still operate within a \textbf{two-stage framework}: we first estimate propensity scores (stage 1) and then optimize outcome models (stage 2). These stages are decoupled; the propensity model does not "know" which observations are hard for the outcome model to predict, and the outcome model is reactive to a fixed propensity distribution. This separation is likely suboptimal. 
A truly optimal estimator would unify the entire process into a \textbf{single-stage} end-to-end optimization problem. Future work should explore architectures where the propensity parameters \(\theta_e\) and outcome parameters \(\theta_\mu\) are learned simultaneously to minimize the final ATE Risk. By allowing the propensity model to adjust its weights specifically to help the outcome model minimize the population risk, we could achieve a fully adaptive, distributionally robust causal estimator, under misspecification models and finite samples.

\bibliography{cite}
\end{document}